\documentclass[]{sn-jnl}% APA Reference Style 
\usepackage{graphicx}%
\usepackage{multirow}%
\usepackage{amsmath,amssymb,amsfonts}%
\usepackage{amsthm}%
\numberwithin{equation}{section}
\usepackage{booktabs}

\usepackage{mathrsfs}%
\usepackage[title]{appendix}%
\usepackage{xcolor}%
\usepackage{textcomp}%
\usepackage{manyfoot}%
\usepackage{booktabs}%
\usepackage{algorithm}%
\usepackage{algorithmicx}%
\usepackage{algpseudocode}%
\usepackage{upgreek}
\usepackage{listings}%
\usepackage{lineno} % Add line numbers package

\usepackage{geometry}
\usepackage{tikz}
\usepackage{caption}
\usepackage{adjustbox}
\usetikzlibrary{shapes.geometric, arrows}
\tikzstyle{startstop} = [rectangle, rounded corners, 
minimum width=3cm, 
minimum height=1cm,
text centered, 
draw=black, 
fill=red!30]

\tikzstyle{io} = [trapezium, 
trapezium stretches=true, % A later addition
trapezium left angle=70, 
trapezium right angle=110, 
minimum width=3cm, 
minimum height=1cm, text centered, 
draw=black, fill=blue!30]

\tikzstyle{process} = [rectangle, 
minimum width=3cm, 
minimum height=1cm, 
text centered, 
text width=3cm, 
draw=black, 
fill=orange!30]

\tikzstyle{decision} = [diamond, 
minimum width=3cm, 
minimum height=2cm,
text centered, 
draw=black, 
fill=green!30]
\tikzstyle{arrow} = [thick,->,>=stealth]
\usepackage[style=numeric, backend=biber]{biblatex}
\theoremstyle{thmstyleone}%
\newtheorem{theorem}{Theorem}[section]

\newtheorem{corollary}[theorem]{Corollary}
\newtheorem{proposition}[theorem]{Proposition}

\theoremstyle{thmstyletwo}%

\theoremstyle{thmstylethree}%

\begin{document}

%\linenumbers

\title[Article Title]{Trend-Adjusted Time Series Models with an Application to Gold Price Forecasting}


\author[a]{\fnm{Sina} \sur{Kazemdehbashi}}

\affil[a]{\orgdiv{} \orgname{hm3369@wayne.edu}}


\abstract{
Time series data play a critical role in various fields, including finance, healthcare, marketing, and engineering. A wide range of techniques—from classical statistical models to neural network-based approaches such as Long Short-Term Memory (LSTM)—have been employed to address time series forecasting challenges. In this paper, we reframe time series forecasting as a two-part task: (1) predicting the trend (directional movement) of the time series at the next time step, and (2) forecasting the quantitative value at the next time step. The trend can be predicted using a binary classifier, while quantitative values can be forecasted using models such as LSTM and Bidirectional Long Short-Term Memory (Bi-LSTM). Building on this reframing, we propose the Trend-Adjusted Time Series (TATS) model, which adjusts the forecasted values based on the predicted trend provided by the binary classifier. We validate the proposed approach through both theoretical analysis and empirical evaluation. The TATS model is applied to a volatile financial time series—the daily gold price—with the objective of forecasting the next day’s price. Experimental results demonstrate that TATS consistently outperforms standard LSTM and Bi-LSTM models by achieving significantly lower forecasting error. In addition, our results indicate that commonly used metrics such as MSE and MAE are insufficient for fully assessing time series model performance. Therefore, we also incorporate trend detection accuracy, which measures how effectively a model captures trends in a time series.
}


\maketitle

\section{Introduction}\label{sec:intro}
Time series data is widely used across various domains such as finance \parencite{sezer2020financial}, healthcare \parencite{morid2023time}, marketing \parencite{dekimpe2000time}, and environmental science \parencite{kaur2023autoregressive}, where accurate forecasting is crucial for informed decision-making and strategic planning. However, challenges like non-stationarity, high volatility, and noise present significant obstacles. As a result, developing more robust and effective modeling techniques remains a critical and ongoing area of research.

Unlike much of the existing literature, which categorize trend prediction in time series as a separate problem from time series forecasting (see for example \cite{kuo2024hybrid,yin2023research,Mahato2014,sadorsky2021predicting,chen2021novel,chen2021constructing}), this work proposes a unified perspective. We reframe trend prediction as a subproblem within the broader scope of time series forecasting. Based on this viewpoint, we propose the Trend-Adjusted Time Series (TATS) model, which includes three main components. First, a \emph{trend predictor}, which employs a binary classifier to predict whether the time series will increase or decrease at the next time step. Second, a \emph{value forecaster}, which utilizes a linear model (e.g., ARIMA) or a nonlinear model (e.g., LSTM) to forecast the actual value of the time series at the next time step. Third, an \emph{adjustment function} (as defined in (\ref{eq:basic-adjustement-definition})), which adjusts the forecasted value based on the predicted trend and produces the final output of the TATS model.

In this study, we investigate the forecasting of gold prices—a volatile financial time series often considered a challenging prediction task \parencite{9776141}—underscoring the need for accurate and robust time series models. Gold prices are selected not only for their significance in economic decision-making and financial markets but also for the complexity they present in forecasting. In recent years, deep learning models such as LSTM and Bi-LSTM have been widely used as nonlinear time series models to address this challenge \parencite{livieris2020cnn,foroutan2024deep}. Our experimental results demonstrate that the proposed TATS model outperforms both LSTM and Bi-LSTM models on this time series by achieving lower forecasting errors.

The organization of this paper is as follows: Section \ref{sec: literature-review} reviews the relevant literature. Section \ref{sec:methodology} presents the Trend-Adjusted Time Series (TATS) model along with its theoretical analysis. In Section \ref{sec:experimantal_results}, we apply Logistic Regression, Naive Bayes, K-Nearest Neighbor (KNN), Support Vector Machine (SVM), Decision Tree, Random Forest, and XGBoost to predict the trend of gold prices, followed by LSTM, Bi-LSTM, and the proposed TATS model for forecasting the next-day price. Section \ref{sec:conclusion} summarizes the main findings and offers insights and recommendations for future research.

\begin{figure}
    \centering
    \includegraphics[width=0.7\linewidth]{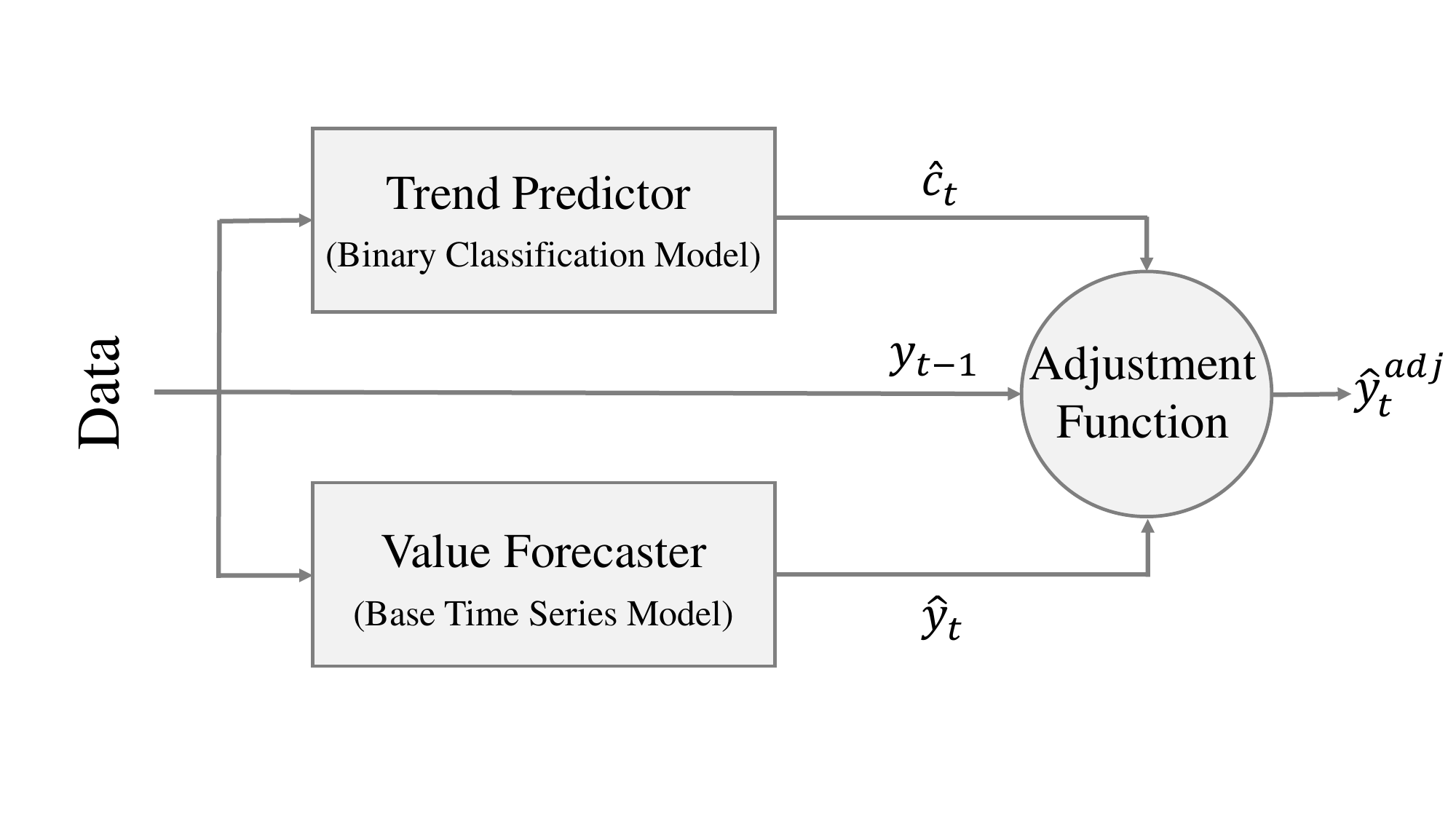}
    \caption{Trend-Adjusted Time Series (TATS) Model}
  \label{fig:TATS_model}
\end{figure}

\section{Literature Review}\label{sec: literature-review}

Forecasting time series data is generally viewed in the literature as involving two key challenges. The first is predicting the trend for future time steps, which can be formulated as a classification problem. The second is forecasting the actual values for upcoming time steps. Accordingly, this section is divided into two parts: first, we review works on trend prediction, and second, we examine studies on value forecasting.

To address the trend prediction challenge, \cite{Mahato2014} employed ensemble models to predict the trend of the gold and silver prices. For gold, their results indicated $85\%$ and $79\%$ accuracy scores by stacking and hybrid bagging, respectively.
\cite{sadorsky2021predicting} employed tree-based classifiers to forecast the price direction of gold and silver. Their results indicated accuracy scores between $85\%$ and $90\%$ for a 20-day forecast horizon using tree bagging, stochastic gradient boosting, and random forests. Also,    \cite{basher2022forecasting} employed tree-based machine learning classifiers and traditional logit econometric models to predict the price direction of gold and Bitcoin. Their results showed that, for 5-day predictions, the accuracy of bagging and random forest models ranged between $75\%$ and $80\%$. For 10-day and 20-day, these models achieved accuracies above $85\%$. In contrast, the logit models demonstrated lower accuracy than random forests in predicting the price directions of both gold and Bitcoin. \cite{gong2024research} used ARIMA and LSTM to predict the daily rise and fall of gold prices. Both models achieved low accuracy in this task.

While predicting time series trends remains a challenging task, the main objective is to forecast future values of the time series data. This problem has been extensively investigated in the literature, spanning approaches from classical statistical models to modern neural network models. Such models have also been employed in forecasting gold prices, which is an important topic within financial time series.

\cite{guha2016gold} used ARIMA models for gold price forecasting, and ARIMA(1,1,1) was selected as the final model for the monthly data from November 2003 to January 2016.  \cite{Tripathy_2017} studied the forecasting of gold prices in India from 1990 to 2015. They applied the ARIMA model to monthly data and showed that ARIMA(0,1,1) provided the best fit. ARIMA has also been used for forecasting other financial data, such as Bitcoin. \cite{hafezi2018forecasting} proposed a hybrid model in which a meta-heuristic algorithm (BAT algorithm) was employed to reinforce the training phase of the artificial neural network (ANN). They referred to their proposed method as BNN (BAT-Neural Network). They showed that the BNN model outperforms both conventional and modern forecasting models, such as the ANN and ARIMA, in gold price forecasting. \cite{krishna2019forecasting}  compared the ARIMA model and feed-forward neural networks (FFNN) for forecasting the daily gold price in India. They used data from 2014 to 2018 and reported that the FFNN model outperformed ARIMA. \cite{madhika2023gold} used ARIMA(0,1,1) and LSTM to predict the daily gold price from 2002 to 2022. They showed that LSTM outperforms the ARIMA model. \cite{nguyen2023machine} investigated gold price prediction using diverse machine learning algorithms such as support vector machines, decision trees, and neural networks, demonstrating their superior performance over traditional statistical methods. The support vector machine emerged as the most accurate, with low error rates and high $R^2$ scores. In \cite{rashidi2025predicting}, the goal of accurately forecasting Gold prices with a minimum of error and a maximum of accuracy is achieved by defining a hybrid modeling approach. This hybrid model combines ARIMA-based time-series forecasting of commodities that drive gold prices, such as oil prices and the S\&P 500, as well as the EUR/USD exchange rate, with additional technical analyses that are reflective of market behavior. Stepwise regression is utilized to determine which of these variables are most significant, and these are then used as inputs within an artificial neural network. This combined ARIMA-Stepwise Regression-Neural Network model is then utilized to accurately predict gold prices. This study showed that the combined model is significantly more accurate than both basic time-series, regression, and stepwise regression models.

The contributions of this paper are summarized as follows. First, we propose the Trend-Adjusted Time Series (TATS) model, which unifies trend prediction and value forecasting into a single framework. We also provide theoretical analysis showing the conditions under which the TATS model yields improved time-series forecasts with lower forecasting error. In addition, we derive a mathematical formula that computes the approximate lower bound on error reduction achievable by using the TATS model. Second, we conduct experiments for one-day-ahead forecasting of gold price data and demonstrate the effectiveness of the TATS model compared with the popular neural network models LSTM and Bi-LSTM.

\section{Methodology}\label{sec:methodology}

\begin{table}
\caption{Notation definitions for the TATS model analysis (Fig.~\ref{fig:TATS_model})} \label{tab:Consistent notation} 
\begin{tabular}{@{}ll@{}}
\toprule
Notations & \\
\midrule
$P(D^B)$ & The probability that the trend predictor correctly predicts the trend of the next time step.\\
$P(D^T)$ &The probability that the value forecaster correctly predicts the trend of the next time step,\\
& meaning that its forecast aligns with the true direction.\\

$\hat{c}_t$ &  Predicted trend (direction) by the trend predictor (binary classifier), where $\hat{c}_t \in \{-1, 1\}$;\\
&  1 indicates an increasing trend and -1 indicates a decreasing trend.\\
$y_t$ & Values of the time series at time $t$.\\
$\hat{y}_t$ & Forecast of $y_t$ by the value forecaster.\\
$\hat{y}^{adj}_t$ & Forecasted value of $y_t$ by the TATS model (output of the TATS model).\\
$\Delta y_t$ & First differencing of time series data at time $t$; i.e. $\Delta y_t=y_t-y_{t-1}$.\\
$l_t$ & The value of the loss (error) at time $t$ for the time series model, calculated as $(\hat{y}_t - y_t)^2$.\\
$\alpha$ &  Hyperparameter at adjustment function which has positive value ($\alpha>0$).\\
$\mathcal{T}$ & Set of time steps corresponding to the time series model’s forecasts.\\
%\color{red}%{$\bar{\mathcal{T}}$} &
%\color{red}{Subset of %$\mathcal{T}$ where for each $t \in \bar{\mathcal{T}}$ we have $\Delta y_t \neq 0$.}\\
\botrule
\end{tabular}

\end{table}

In this section, we introduce the TATS model, illustrated in Fig.~\ref{fig:TATS_model}. The TATS model consists of three main components: a trend predictor (a binary classification model), a value forecaster (a base time series model), and an adjustment function (a modification mechanism). The model inputs the forecast from the value forecaster and the predicted trend (direction) from the trend predictor into the adjustment function. This function uses the provided information to generate an adjusted forecast, with further details presented in Section~\ref{sec:adjustment-function}. In simple terms, the core idea of the TATS model is as follows: if the trend of the forecast generated by the value forecaster matches the trend predicted by the trend predictor, the forecast remains unchanged. Otherwise, the adjustment function modifies the forecast accordingly.

\subsection{Value Forecaster and Trend Predictor}

In this section, we describe the two components of the TATS model. First, we discuss the value forecaster, which serves as the base time series model within TATS. The value forecaster can be either a linear model, such as ARIMA, or a non-linear model, such as LSTM or Bi-LSTM.  
After applying the value forecaster to the time series data, it is crucial to evaluate how well the model captures the correct trend of the data. In other words, we need to assess how accurately the forecasts reflect the directional movements of the time series. To quantify this, we define the \emph{Trend Detection (TD) accuracy}, which measures how effectively a time series model’s forecasts capture the correct trend (direction). TD accuracy is defined as:

\begin{equation}\label{eq:TDA}
\text{TD accuracy} = \frac{\sum_{t\in \mathcal{T}} \sigma_t}{\mathcal{T}}, \ \ \text{where}\ \ \  \sigma_t = \begin{cases}
1 & \text{if } (\hat{y}_t - y_{t-1})(y_t - y_{t-1}) > 0 \\
0 & \text{otherwise}
\end{cases} \ ,
\end{equation}
where, $\hat{y}_t$, $y_t$, and $\mathcal{T}$ are defined in Table \ref{tab:Consistent notation} 
%(\textcolor{red}{It is more accurate to use $\bar{\mathcal{T}}$ as the set of time steps when there exists a trend, i.e. $\Delta y_t \neq 0$; not $\mathcal{T}$ that includes all the time steps. TDA should be defined so that a time series model that predicts $\hat{y}_{t}=y_{t-1}$ has a TDA value of zero}).%
. As TD accuracy approaches 1 (or $100 \%$), the value forecaster effectively captures the trend in the time series data.

The second component of the TATS model is the trend predictor, which uses a binary classifier to predict the direction of movement for the next time step. Common machine learning methods as well as deep learning models can be employed as the trend predictor. According to Proposition~\ref{prop:alpha-small-constant}, if the accuracy of the trend predictor exceeds the TD accuracy of the value forecaster, the TATS model can produce forecasts with lower forecasting errors and demonstrate better overall performance than the value forecaster alone.

\subsection{Adjustment Function Analysis}\label{sec:adjustment-function}

In this section, we analyze the adjustment function employed in the TATS model. For simplicity, in this section we refer to the value forecaster as VF and the trend predictor as TP. Using the notations listed in Table~\ref{tab:Consistent notation}, the adjustment function $f$ is defined as:

\begin{equation}\label{eq:basic-adjustement-definition}
    \hat{y}^{adj}_t =f(\hat{y}_t,\hat{c}_t,y_{t-1},\alpha)= I_t \hat{y}_t + (1 - I_t)(y_{t-1} + \hat{c}_t \alpha),
\end{equation}
where \( I_t \), referred to as the \emph{indicator function}, is defined by:

\begin{equation}\label{eq:Indicator-function}
    I_t = 
    \begin{cases}
        1 & \text{if } (\hat{y}_t - y_{t-1})\hat{c}_t \ge 0, \\
        0 & \text{otherwise}.
    \end{cases}
\end{equation}

The indicator function evaluates whether the trend (direction) of the forecast produced by VF aligns with the trend predicted by TP. Specifically, \( I_t = 1 \) when the two directions agree, and \( I_t = 0 \) when they conflict. Accordingly, the adjustment function~\eqref{eq:basic-adjustement-definition} behaves as follows: if \( I_t = 1 \), it returns the unmodified forecast \( \hat{y}_t \) (generated by VF); if \( I_t = 0 \), it adjusts the previous value \( y_{t-1} \) in the direction predicted by TP ($\hat{c}_t$), scaled by a hyperparameter \( \alpha \), yielding the adjusted value \( y_{t-1} + \hat{c}_t \alpha \).

Based on the design of the TATS model (see Fig.~\ref{fig:TATS_model}), four possible scenarios can arise, as illustrated in Fig.~\ref{fig:direction-prediction}. In the first scenario (denoted as \( S_1 \)), TP correctly predicts the trend of the next time step, and the VF's forecast also lies in the correct trend. Thus, the TATS model makes no adjustment, as shown in Fig.~\ref{fig:direction-prediction}(a).
In the second scenario (\( S_2 \)), the VF's forecast lies in the correct trend, while TP predicts the direction incorrectly. In this case, the TATS model activates its adjustment function, which may either reduce or increase the forecasting error. For the purpose of analysis, we assume the worst case, in which the adjustment increases the error of the VF's forecast, as shown in Fig.~\ref{fig:direction-prediction}(b).
In the third scenario (\( S_3 \)), the VF's forecast lies in the wrong trend, and TP also predicts the wrong trend. Therefore, the TATS model does not apply any adjustment (Fig.~\ref{fig:direction-prediction}(c)). In the fourth scenario (\( S_4 \)), TP correctly predicts the trend, while the VF's forecast lies in the wrong trend. In this case, the TATS model adjusts the forecast, reducing the forecasting error of VF, as illustrated in Fig.~\ref{fig:direction-prediction}(d).

In the formulation of the adjustment function (\ref{eq:basic-adjustement-definition}), \( \alpha \) can be interpreted in several ways. First, it can be set as a small constant (close to zero) that simply indicates the direction of adjustment—this is referred to as the basic version of the adjustment function. Second, \( \alpha \) can be treated as a tunable hyperparameter. In the following section, we analyze the adjustment function under the first scenario.

\subsubsection{$\alpha$ as a small positive value}\label{sec:alpha_constant}

In this section, we consider \( \alpha \) in Equation \ref{eq:basic-adjustement-definition} as a positive and sufficiently small constant. The following proposition provides a condition under which the adjustment function reduces the forecasting error of the value forecaster.

\begin{proposition}\label{prop:alpha-small-constant}
Let \( \alpha \) in the adjustment function (\ref{eq:basic-adjustement-definition}) be a positive and sufficiently small constant such that \( y_{t-1} + \hat{c}_t \alpha \approx y_{t-1} \). Then, the TATS model reduces the forecasting error of its value forecaster if \( P(D^B) > P(D^T) \).
\end{proposition}

\begin{proof}

We define \( \delta_t^{S_i} \) as the change in loss (error) under scenario \( S_i \), for \( i = 1, 2, 3, 4 \), given by:

\begin{equation}
    \delta_t^{S_i} = 
    \begin{cases}
        l_t - (\hat{y}^{adj}_t - y_t)^2 = 0  & \text{for } i=1 \\[8pt]
        l_t - (\hat{y}^{adj}_t - y_t)^2 = l_t - (\Delta y_t)^2  \le 0  & \text{for } i=2 \\[8pt]
        l_t - (\hat{y}^{adj}_t - y_t)^2 = 0  & \text{for } i=3 \\[8pt]
        l_t - (\hat{y}^{adj}_t - y_t)^2 = l_t - (\Delta y_t)^2  \ge 0  & \text{for } i=4
    \end{cases}
\end{equation}

Note that in scenarios \( S_2 \) and \( S_4 \), we have \( \hat{y}^{\text{adj}}_t \approx y_{t-1} \) and \( (\hat{y}_t^{\text{adj}} - y_t)^2 \approx (\Delta y_t)^2 \), since \( \alpha \) is a sufficiently small constant. Moreover, for all \( t \), we assume the worst case scenario for \( S_2 \), meaning \( \delta_t^{S_2} \leq 0 \).

We define \( \Delta l_t \) as the expected change in loss at time \( t \). Taking into account the probability of each scenario, we obtain:
\begin{equation}
 \Delta l_t=\sum_{i=1}^{4}\delta_t^{S_i}P(S_i)=|l_t-(\Delta y_t)^2|P(D^B)(1-P(D^T))-|l_t-(\Delta y_t)^2|(1-P(D^B))P(D^T),
\end{equation}
where 
\begin{equation}
    P(S_i) = 
    \begin{cases}
        P(D^B)P(D^T) & \text{for } i=1 \\[8pt]
        (1-P(D^B))P(D^T)  & \text{for } i=2 \\[8pt]
        (1-P(D^B))(1-P(D^T))  & \text{for } i=3 \\[8pt]
        P(D^B)(1-P(D^T))  & \text{for } i=4
    \end{cases}\ \ \ \ .
\end{equation}
By taking the expectation over 
$t$, we obtain the total expected change in loss as:
\begin{equation}\label{eq:expected-loss-change-zero-alpha}
    E_{t\in \mathcal{T}}(\Delta l_t)=E(|l_t-(\Delta y_t)^2|)(P(D^B)(1-P(D^T))-(1-P(D^B))P(D^T)). 
\end{equation}
Therefore, to reduce the value forecaster's loss, the expected loss change \( E_{t\in \mathcal{T}}(\Delta l_t) \) must be positive. This condition requires:
\[
P(D^B)(1 - P(D^T)) > (1 - P(D^B)) P(D^T),
\]
which simplifies as follows:
\begin{equation}
    \begin{split}
        & P(D^B)(1 - P(D^T)) > (1 - P(D^B)) P(D^T) \\
        & \Rightarrow P(D^B) - P(D^B) P(D^T) > P(D^T) - P(D^B) P(D^T) \\
        & \Rightarrow P(D^B) > P(D^T).
    \end{split}
\end{equation}

Thus, the condition \( P(D^B) > P(D^T) \) ensures that the TATS model yields a lower forecasting error than its value forecaster.
\end{proof}

Proposition~\ref{prop:alpha-small-constant} shows that, under certain conditions, the TATS model can reduce the forecasting error of its value forecaster. The following corollary provides an approximate lower bound for this error reduction.

\begin{corollary}\label{corollary:LBapproximate-zero-alpha}
Let \( \alpha \) in the adjustment function (\ref{eq:basic-adjustement-definition}) be a positive and sufficiently small constant such that \( y_{t-1} + \hat{c}_t \alpha \approx y_{t-1} \), and assume \( P(D^B) > P(D^T) \). Then,
\begin{equation}
E\left( \left| l_t - (\Delta y_t)^2 \right| \right)\big(P(D^B) - P(D^T)\big)
\end{equation}
provides an approximate lower bound on the forecasting error reduction achieved by the TATS model.
\end{corollary}

\begin{proof}
This result follows directly from Equation \ref{eq:expected-loss-change-zero-alpha}.
\end{proof}

\begin{figure}
    \centering
    \includegraphics[width=0.9\linewidth]{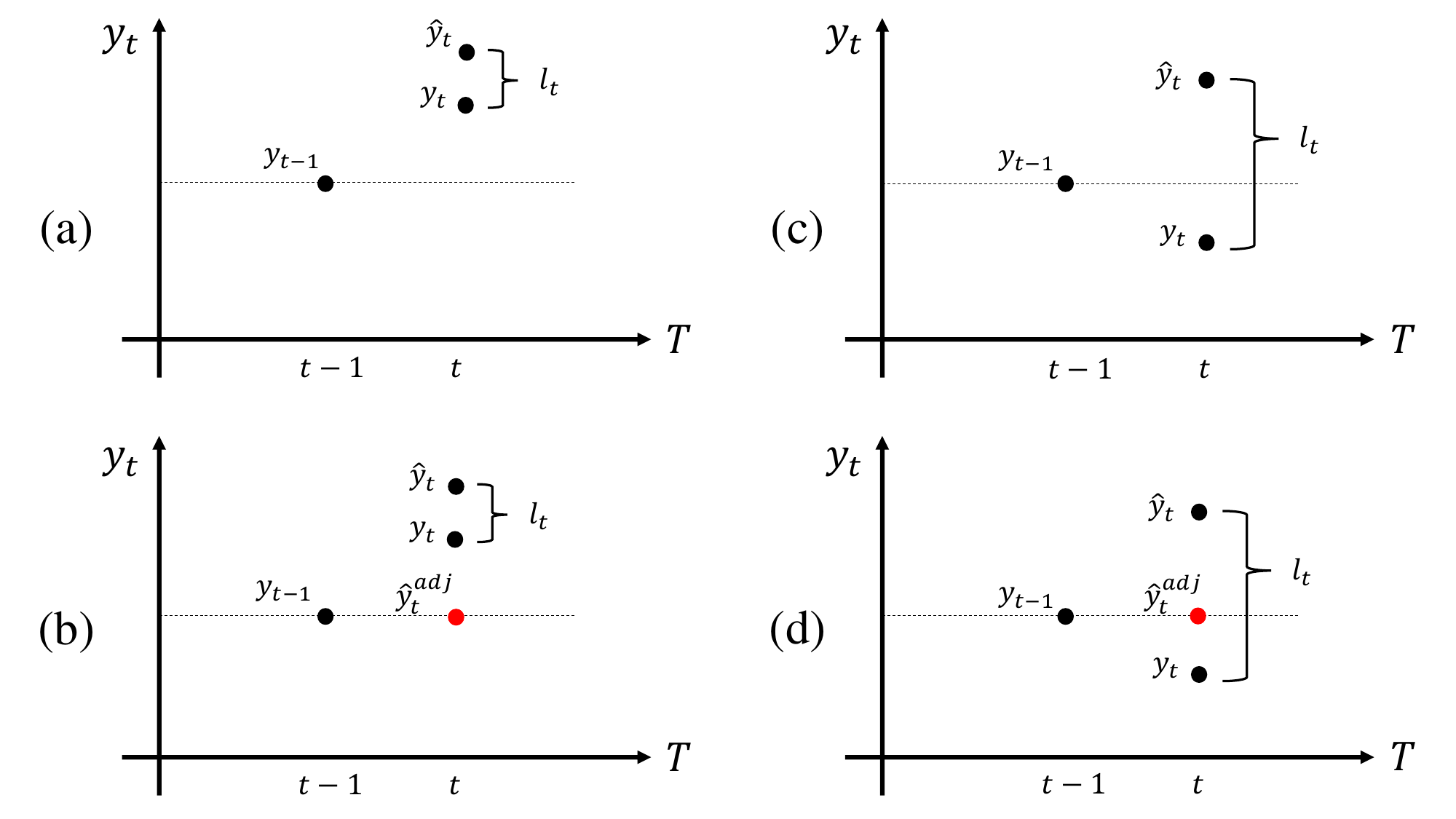}
\caption{(a) Both the trend predictor (TP) and value forecaster (VF) predict the trend correctly; (b) VF predicts the trend correctly, but TP predicts incorrectly; (c) Both TP and VF predict the trend incorrectly; (d) VF predicts the trend incorrectly, while TP predicts correctly. In all cases, \( \alpha \) is positive and sufficiently small ($y_{t-1} + \hat{c}_t \alpha \simeq y_{t-1}$).}

  \label{fig:direction-prediction}
\end{figure}

\subsubsection{$\alpha$ as a hyperparameter}\label{sec:alpha-hyperparameter}

In Section \ref{sec:alpha_constant}, $\alpha$ is treated as a small positive value (close to zero). However, in practice, it can be considered a hyperparameter and assigned different values. As shown in Section \ref{sec:experimantal_results}, running the TATS model with various $\alpha$ values reveals that for some choices of $\alpha$, the model achieves lower forecasting error than when $\alpha$ is close to zero.

\section{Experiment Results}\label{sec:experimantal_results}

This section focuses on the daily gold price in 2023, as a financial time series consisting of 250 observations. We assign $70\%$ of the data to the training set and $30\%$ to the test set. The main purpose is to forecast the value of gold for the next day (next time step). This section is divided into three subsections.

First, we focus on trend prediction, aiming to predict the direction (increasing or decreasing) of the gold price for the next day. In this subsection, we use several binary classifiers and compare their performance to identify the model with the highest accuracy, which is then chosen as the trend predictor for the TATS model. Second, we apply two common neural network models, LSTM and Bi-LSTM, to forecast the value of the gold price for the next day. Finally, we construct two TATS models and demonstrate that both outperform LSTM and Bi-LSTM in terms of forecasting accuracy.

\subsection{Trend Predictor (Binary Classifiers)}

In this section, the goal is to predict the trend of the daily gold price for the next day. To achieve this, we use not only the gold price itself but also additional features that influence the gold price \cite{kristjanpoller2015gold}. These features include the FTSE (London Stock Market Index), the Dow Jones Index, the S\&P 500 Index, the USD/EUR exchange rate, the USD/JPY exchange rate, and the oil price, all obtained from Yahoo Finance.

Several machine learning models are considered, including Logistic Regression, Naive Bayes (NB), K-Nearest Neighbors (KNN), Support Vector Machine (SVM), Decision Tree, Random Forest, and XGBoost. We use 5-fold cross-validation for hyperparameter tuning, and we report three accuracy measures for each model: training accuracy, cross-validation accuracy, and test accuracy.

\begin{table}[h]
  \caption{Results for Trend Predictors (Binary Classifiers)}
  \label{tab:HBC-results}
  \centering
  \begin{tabular}{lcccc}
    \toprule
    
    No.     &  Model & Train accuracy(\%) & Cross-validation accuracy (\%) &Test accuracy(\%)   \\
    \midrule
  1& Logistic Regression &64.57 & 61.71  & 50.66 \\
   2& NB & 66.85 & 58.85 & 56.00 \\
  3 & KNN &68.57 & 58.28  & 52.00 \\ 
     
   4 &SVM& 65.14 & 57.70  & 53.33\\
    
    4& Decision Tree & 76.57 & 57. 71&46.66\\

   5& Random Forest &  89.14 & 60.57&57.33\\

   6 & XGBoost &$75.14$ & 65.71 & 58.66\\

    \bottomrule
  \end{tabular}
\end{table}

The results in Table \ref{tab:HBC-results} show that most models perform poorly, with test accuracy close to $50\%$. This indicates that the gold price is volatile and challenging to model, making trend prediction for this time series a difficult task. Among all models, XGBoost demonstrates the best performance, achieving $75.14\%$ training accuracy and $58.66\%$ test accuracy. Therefore, in Section \ref{sec:experiment-TATS}, we use XGBoost as the trend predictor for the TATS models.

\subsection{Value Forecaster (Base Time Series Models)}
In this section, we apply LSTM and Bi-LSTM models to the daily gold price data to forecast the price for the next day (next time step). For both models, the lag is set to two, meaning they use data from the two preceding time steps to forecast the price at each step.

\begin{table}[h]
  \caption{Results for Value Forecasters (Base Time Series Models)}
  \label{tab:value_forecator-results}
  \centering
  \begin{tabular}{lccccc}
    \toprule
    
    Model   &train/test &TDA&  MSE & MAE & MAPE   \\
       
    \midrule
  LSTM  &train&52.36 $\%$ & 324.47  & 13.78 & 0.71 \\
     &test&51.29 $\%$ &399.31  & 16.05& 0.82  \\
 
    \cmidrule(lr){1-6}  
    Bi-LSTM  &train&52.41 $\%$ & 300.94  & 13.21 & 0.68 \\
     &test&51.20 $\%$ &371.92  & 15.35& 0.78  \\

    \bottomrule
  \end{tabular}
  {\scriptsize Note: TDA represents the trend detection accuracy, as defined in (\ref{eq:TDA}).
}
\end{table}

Results in Table \ref{tab:value_forecator-results} show that the Bi-LSTM outperforms the LSTM; however, both models exhibit very poor performance in trend detection (TD accuracy is close to \(50\%\)). Consequently, about half of the forecasts produced by these models are in the wrong trend (direction), indicating that they are not reliable tools for decision-making. In other words, common metrics such as MAE, MSE, and MAPE are not sufficient to fully evaluate a time series model. Assessing how well a model captures the trend or directional movement of a time series is also crucial; therefore, measuring TD accuracy is important for each time series model.

\subsection{TATS model}\label{sec:experiment-TATS}
In this section, we build two TATS models and evaluate their performance on daily gold price data.

First, we construct a TATS model using an LSTM as the value forecaster and XGBoost as the trend predictor. After training the LSTM on the data, we obtain
\[
E\bigl(|l_t - (\Delta y_t)^2|\bigr) = 180.45,
\]
computed on the training set. The LSTM achieves a trend detection (TD) accuracy of \(52.36\%\) on the training set (an estimate of $P(D^T)$), while XGBoost attains a training accuracy of \(75.14\%\) (an estimate of $P(D^B)$). Since $P(D^B)>P(D^T)$, Proposition~\ref{prop:alpha-small-constant} indicates that the TATS model can achieve a lower forecasting error than the LSTM alone. According to Corollary~\ref{corollary:LBapproximate-zero-alpha}, applying the TATS model (with a small value of $\alpha$) yields a forecasting model whose training-set MSE is at least
\[
E\bigl(|l_t - (\Delta y_t)^2|\bigr) \bigl(P(D^B)-P(D^T) \bigr)=180.45 \times (0.7514 - 0.5236) = 41.10
\]
lower than that of the LSTM (its value forecaster). Results in Table \ref{tab:TATS-LSTM-results} show that the TATS model with \(\alpha = 0.1\) achieves an MSE of 250.20, which is 74.27 lower than the MSE of the LSTM model. This empirical result aligns with the theoretical guarantee established in Corollary~\ref{corollary:LBapproximate-zero-alpha}.
Additional results indicate that, except for \(\alpha = 20\), all other values of \(\alpha\) yield better forecasting performance than the LSTM.

\begin{table}[h]
  \caption{Results for TATS model with LSTM (value forecastor) and XGBoost (trend predictor)}
  \label{tab:TATS-LSTM-results}
  \centering
  \begin{tabular}{lccccccc}
    \toprule
    
    Model   &train/test &TDA&  MSE & MAE & MAPE & Diff & R-Diff    \\
       
    \midrule
  LSTM  &train&52.36 $\%$ & 324.47  & 13.78 & 0.71 &-&-\\
     &test&51.29 $\%$ &399.31  & 16.05& 0.82 & -&-\\
 
    \cmidrule(lr){1-8}  
    
    TATS (LSTM, XGBoost, $\alpha=0.1$)  &train&75.14 $\%$ &250.20  & 11.63 & 0.60 & 74.27&22.88$\%$\\
     &test&58.66 $\%$ &332.97  & 14.10 & 0.72 & 66.34&16.61$\%$\\
 
    \cmidrule(lr){1-8}  

    TATS (LSTM, XGBoost, $\alpha=1$)  &train&75.14 $\%$ &245.44  & 11.42& 0.59 &79.03&24.35$\%$\\
     &test&58.66 $\%$ & 331.52 & 14.02 &  0.71 & 67.79&16.97$\%$\\
 
    \cmidrule(lr){1-8}  

    TATS (LSTM, XGBoost, $\alpha=2$)  &train&75.14 $\%$ &241.24  & 11.27 & 0.58 & 83.23&25.65$\%$\\
     &test&58.66 $\%$ & {\color{red}330.95}  & 13.95 & 0.71 & 68.36&17.11$\%$\\
 
    \cmidrule(lr){1-8}

     TATS (LSTM, XGBoost, $\alpha=5$)  &train&75.14 $\%$ & {\color{red}235.51}  & 11.12 & 0.57 &88.96&27.41 $\%$\\
     &test&58.66 $\%$ & 335.80  & 13.85 &0.70 & 63.51&15.90 $\%$\\
 
    \cmidrule(lr){1-8}  
     TATS (LSTM, XGBoost, $\alpha=10$)  &train&75.14 $\%$ &248.84  & 11.75 &0.61 & 75.63 &23.30 $\%$\\
     &test&58.66 $\%$ & 365.74  & 14.40 &0.73 & 33.57 &8.40 $\%$\\

    \cmidrule(lr){1-8}  
     TATS (LSTM, XGBoost, $\alpha=20$) &train&75.14 $\%$ & 361.35  & 15.46 &  0.80 & $-36.88$ &$-11.36 \%$\\
     &test&58.66 $\%$ & 507.64  & 17.42 &   0.88 & $-108.33$ & $-27.12 \%$\\

    \bottomrule
  \end{tabular}

{\scriptsize Note: TDA represents the trend detection accuracy, as defined in (\ref{eq:TDA}). \textit{Diff} for each model denotes the difference between the model's MSE and that of the LSTM. \textit{R-Diff} is calculated as \textit{Diff} divided by the LSTM's MSE.}
\end{table}

Second, we construct a TATS model using a Bi-LSTM as the value forecaster and XGBoost as the trend predictor. Training the Bi-LSTM on the data yields
\[
E\bigl(|l_t - (\Delta y_t)^2|\bigr) = 147.70,
\]
computed on the training set. The Bi-LSTM performs poorly in capturing the gold price trend, achieving a trend detection (TD) accuracy of 52.41\%, which serves as an estimate of \(P(D^T)\).

Based on Corollary~\ref{corollary:LBapproximate-zero-alpha} and the training accuracy of 75.14\% (an estimate of \(P(D^B)\)) achieved by XGBoost, the TATS model with a small value of \(\alpha\) is guaranteed to reduce the training-set MSE by at least
\[
E\bigl(|l_t - (\Delta y_t)^2|\bigr)\,\bigl(P(D^B)-P(D^T)\bigr) = 147.70 \times (0.7514 - 0.5241) = 33.57
\]
compared to the Bi-LSTM (its value forecaster). Based on results in Table \ref{tab:TATS-Bi-LSTM-results}, the TATS model with \(\alpha = 0.1\) achieves a training-set MSE of 245.03, which is 55.91 lower than that of the Bi-LSTM model. This result is consistent with the theoretical lower bound provided by Corollary~\ref{corollary:LBapproximate-zero-alpha}. 
The TATS model achieves its lowest MSE on the test set with \(\alpha = 2\) (320.75) and on the training set with \(\alpha = 5\) (233.91). Based on MAE, the model with \(\alpha = 5\) performs best on both the training and test sets. Note that the TD accuracy of a TATS model corresponds directly to the prediction accuracy of its trend predictor.

\begin{table}[h]
  \caption{Results for TATS model with Bi-LSTM (value forecastor) and XGBoost (trend predictor)}
  \label{tab:TATS-Bi-LSTM-results}
  \centering
  \begin{tabular}{lcccccccc}
    \toprule
    
    Model   &train/test &TDA&  MSE & MAE & MAPE & Diff & R-Diff    \\
       
    \midrule
  Bi-LSTM  &train&52.41 $\%$ & 300.94  & 13.21 & 0.68 &-&-\\
     &test&51.20 $\%$ &371.92  & 15.35& 0.78 & -&-\\
 
    \cmidrule(lr){1-8}  
    
    TATS (Bi-LSTM, XGBoost, $\alpha=0.1$)  &train&75.14 $\%$ &245.03  & 11.50 & 0.59 & 55.91&$18.57 \%$\\
     &test&58.66 $\%$ &323.51  & 13.81 & 0.70& 48.41&$13.01 \%$\\
 
    \cmidrule(lr){1-8}  

    TATS (Bi-LSTM, XGBoost, $\alpha=1$)  &train&75.14 $\%$ &240.97 &  11.31& 0.58 &59.97&$19.92 \%$\\
     &test&58.66 $\%$ & 321.73 & 13.71 &  0.70 & 50.19&$13.49 \%$\\
 
    \cmidrule(lr){1-8}  

    TATS (Bi-LSTM, XGBoost, $\alpha=2$)  &train&75.14 $\%$ &237.52  & 11.18 & 0.58 & 63.42&$21.07 \%$\\
     &test&58.66 $\%$ & {\color{red}320.75}  & 13.63 &0.69 & 51.17&13.75$\%$\\
 
    \cmidrule(lr){1-8}

     TATS (Bi-LSTM, XGBoost, $\alpha=5$)  &train&75.14 $\%$ & {\color{red}233.91}  & 11.07 &  0.57 &67.03&$22.27\%$\\
     &test&58.66 $\%$ & 324.23  & 13.50 &0.68 & 47.69&$12.82 \%$\\
 
    \cmidrule(lr){1-8}  
     TATS (Bi-LSTM, XGBoost, $\alpha=10$)  &train&75.14 $\%$ &250.31  & 11.78 & 0.61 & $50.63$ &$16.82 \%$\\
     &test&58.66 $\%$ & 351.37  & 13.98 &0.71& 20.55 &$5.52 \%$\\

    \cmidrule(lr){1-8}  
     TATS (Bi-LSTM, XGBoost, $\alpha=20$) &train&75.14 $\%$ &367.22 & 15.61 &  0.81 & $-66.28$ &$-22.02 \%$\\
     &test&58.66 $\%$ &  485.63  & 16.86 &  0.86 & $-113.71$ &$-30.57 \%$\\

    \bottomrule
  \end{tabular}

{\scriptsize Note: TDA represents the trend detection accuracy, as defined in (\ref{eq:TDA}). \textit{Diff} for each model denotes the difference between the model's MSE and that of the Bi-LSTM. \textit{R-Diff} is calculated as \textit{Diff} divided by the Bi-LSTM's MSE.}

\end{table}

\section{Conclusion}\label{sec:conclusion}
Time series data appears in various domains, including healthcare, finance, manufacturing, and marketing. Over the past years, two main challenges in time series analysis have been studied. The first is trend prediction, which focuses on identifying the directional movement of the time series for the next time step. The second is value forecasting, which aims to forecast the actual value of the time series at the next time step. In this work, we unified these two challenges within a single framework and proposed the TATS model. The TATS model consists of three components: (1) a trend predictor, which is a binary classifier that predicts the directional movement of the time series for the next step; (2) a value forecaster, which employs common time series models such as ARIMA, LSTM, and Bi-LSTM to forecast the value of the time series at the next time step; and (3) an adjustment function, which modifies the forecast generated by the value forecaster based on the trend predicted by the trend predictor.
We proved that the TATS model can achieve lower forecasting errors compared to its value forecaster under a certain condition. In addition, we provided a mathematical formula that approximates the lower bound of the forecasting error reduction achievable by the TATS model.

To evaluate the effectiveness of the TATS model, we applied it to the daily gold price time series, which is known to be particularly challenging. We used data from 2023 and first applied popular neural network models, LSTM and Bi-LSTM. The results showed that these models could achieve low forecasting errors; however, they performed poorly in capturing the trend of the gold price series. In other words, approximately half of the forecasts made by LSTM and Bi-LSTM were in the wrong trend. These results indicate that common metrics like MSE and MAE are not sufficient to fully assess the performance of a time series model. Therefore, we also used trend detection accuracy to evaluate how well the models capture the trends (directional movements) of the time series. To obtain a time series model with lower forecasting error and higher trend detection accuracy compared to LSTM and Bi-LSTM, we constructed two TATS models. The first model used an LSTM as the value forecaster and XGBoost as the trend predictor, while the second model employed a Bi-LSTM as the value forecaster and XGBoost as the trend predictor. The first TATS model outperformed the LSTM baseline, achieving a $17.11\%$ reduction in forecasting error on the test set and an MSE of 330.95. Similarly, the second TATS model outperformed the Bi-LSTM baseline, reducing forecasting error by $13.75\%$ and achieving an MSE of 320.75.

Several aspects of this work can be extended. First, the trend predictor can be improved by incorporating large language models and news data. Second, the adjustment function can be replaced with a more flexible formulation. Third, while our focus in this paper was on gold prices as a case study in financial time series, future studies can apply this approach to other areas, such as healthcare, where even small improvements in forecasting can make a significant difference.

\printbibliography

@article{Tripathy_2017, title={Forecasting Gold Price with Auto Regressive Integrated Moving Average Model}, volume={7}, url={https://www.econjournals.com/index.php/ijefi/article/view/4873}, number={4}, journal={International Journal of Economics and Financial Issues}, author={Tripathy, Naliniprava}, year={2017}, month={Jul.}, pages={324–329} }

@article{hafezi2018forecasting,
  title={Forecasting gold price changes: Application of an equipped artificial neural network},
  author={Hafezi, Reza and Akhavan, Amir},
  journal={AUT Journal of Modeling and Simulation},
  volume={50},
  number={1},
  pages={71--82},
  year={2018},
  publisher={Amirkabir University of Technology}
}

@article{krishna2019forecasting,
  title={Forecasting of daily prices of gold in india using ARIMA and FFNN Models},
  author={Krishna, K Murali and Reddy, N Konda and Sharma, M Raghavendra},
  journal={International journal of engineering and advanced technology (IJEAT)},
  volume={8},
  number={3},
  pages={516--521},
  year={2019}
}

@article{kristjanpoller2015gold,
  title={Gold price volatility: A forecasting approach using the Artificial Neural Network--GARCH model},
  author={Kristjanpoller, Werner and Minutolo, Marcel C},
  journal={Expert systems with applications},
  volume={42},
  number={20},
  pages={7245--7251},
  year={2015},
  publisher={Elsevier}
}

@inproceedings{nguyen2023machine,
  title={Machine Learning Algorithms for Gold Price Prediction},
  author={Nguyen, Duyen Mai Thi and Debnath, Narayan C and Quach, Luyl-Da and Nguyen, Vinh Dinh},
  booktitle={International Conference on Advanced Intelligent Systems and Informatics},
  pages={212--220},
  year={2023},
  organization={Springer}
}

@article{guha2016gold,
  title={Gold price forecasting using ARIMA model},
  author={Guha, Banhi and Bandyopadhyay, Gautam},
  journal={Journal of Advanced Management Science},
  volume={4},
  number={2},
  year={2016}
}

@article{madhika2023gold,
  title={Gold Price Prediction Using the ARIMA and LSTM Models},
  author={Madhika, Yudha Randa and Kusrini, Kusrini and Hidayat, Tonny},
  journal={Sinkron: jurnal dan penelitian teknik informatika},
  volume={7},
  number={3},
  pages={1255--1264},
  year={2023}
}

@inproceedings{gong2024research,
  title={Research on gold price forecasting based on lstm and linear regression},
  author={Gong, Weichen},
  booktitle={SHS Web of Conferences},
  volume={181},
  pages={02005},
  year={2024},
  organization={EDP Sciences}
}

@article{basher2022forecasting,
  title={Forecasting Bitcoin price direction with random forests: How important are interest rates, inflation, and market volatility?},
  author={Basher, Syed Abul and Sadorsky, Perry},
  journal={Machine Learning with Applications},
  volume={9},
  pages={100355},
  year={2022},
  publisher={Elsevier}
}

@article{sadorsky2021predicting,
  title={Predicting gold and silver price direction using tree-based classifiers},
  author={Sadorsky, Perry},
  journal={Journal of Risk and Financial Management},
  volume={14},
  number={5},
  pages={198},
  year={2021},
  publisher={MDPI}
}

@INPROCEEDINGS{Mahato2014,
  author={Mahato, Pradeep Kumar and Attar, Vahida},
  booktitle={2014 International Conference on Advances in Engineering \& Technology Research (ICAETR - 2014)}, 
  title={Prediction of gold and silver stock price using ensemble models}, 
  year={2014},
  volume={},
  number={},
  pages={1-4},
  doi={10.1109/ICAETR.2014.7012821}}

@article{dekimpe2000time,
  	title={Time-series models in marketing:: Past, present and future},
  	author={Dekimpe, Marnik G and Hanssens, Dominique M},
  	journal={International journal of research in marketing},
  	volume={17},
  	number={2-3},
  	pages={183--193},
  	year={2000},
  	publisher={Elsevier}
  }

@article{kaur2023autoregressive,
  	title={Autoregressive models in environmental forecasting time series: a theoretical and application review},
  	author={Kaur, Jatinder and Parmar, Kulwinder Singh and Singh, Sarbjit},
  	journal={Environmental Science and Pollution Research},
  	volume={30},
  	number={8},
  	pages={19617--19641},
  	year={2023},
  	publisher={Springer}
  }

@article{kuo2024hybrid,
  	title={Hybrid of jellyfish and particle swarm optimization algorithm-based support vector machine for stock market trend prediction},
  	author={Kuo, RJ and Chiu, Tzu-Hsuan},
  	journal={Applied Soft Computing},
  	volume={154},
  	pages={111394},
  	year={2024},
  	publisher={Elsevier}
  }

@article{yin2023research,
  	title={Research on stock trend prediction method based on optimized random forest},
  	author={Yin, Lili and Li, Benling and Li, Peng and Zhang, Rubo},
  	journal={CAAI Transactions on Intelligence Technology},
  	volume={8},
  	number={1},
  	pages={274--284},
  	year={2023},
  	publisher={Wiley Online Library}
  }

@article{chen2021novel,
  	title={A novel graph convolutional feature based convolutional neural network for stock trend prediction},
  	author={Chen, Wei and Jiang, Manrui and Zhang, Wei-Guo and Chen, Zhensong},
  	journal={Information Sciences},
  	volume={556},
  	pages={67--94},
  	year={2021},
  	publisher={Elsevier}
  }

@article{sezer2020financial,
  	title={Financial time series forecasting with deep learning: A systematic literature review: 2005--2019},
  	author={Sezer, Omer Berat and Gudelek, Mehmet Ugur and Ozbayoglu, Ahmet Murat},
  	journal={Applied soft computing},
  	volume={90},
  	pages={106181},
  	year={2020},
  	publisher={Elsevier}
  }

@article{morid2023time,
  	title={Time series prediction using deep learning methods in healthcare},
  	author={Morid, Mohammad Amin and Sheng, Olivia R Liu and Dunbar, Joseph},
  	journal={ACM Transactions on Management Information Systems},
  	volume={14},
  	number={1},
  	pages={1--29},
  	year={2023},
  	publisher={ACM New York, NY}
  }

@article{livieris2020cnn,
  title={A CNN--LSTM model for gold price time-series forecasting},
  author={Livieris, Ioannis E and Pintelas, Emmanuel and Pintelas, Panagiotis},
  journal={Neural computing and applications},
  volume={32},
  number={23},
  pages={17351--17360},
  year={2020},
  publisher={Springer}
}

@article{chen2021constructing,
  title={Constructing a stock-price forecast CNN model with gold and crude oil indicators},
  author={Chen, Yu-Chen and Huang, Wen-Chen},
  journal={Applied Soft Computing},
  volume={112},
  pages={107760},
  year={2021},
  publisher={Elsevier}
}

@INPROCEEDINGS{9776141,
  author={Yang, Jiacheng and De Montigny, Denis and Treleaven, Philip},
  booktitle={2022 IEEE Symposium on Computational Intelligence for Financial Engineering and Economics (CIFEr)}, 
  title={ANN, LSTM, and SVR for Gold Price Forecasting}, 
  year={2022},
  volume={},
  number={},
  pages={1-7},
  doi={10.1109/CIFEr52523.2022.9776141}}

@article{foroutan2024deep,
  title={Deep learning systems for forecasting the prices of crude oil and precious metals},
  author={Foroutan, Parisa and Lahmiri, Salim},
  journal={Financial Innovation},
  volume={10},
  number={1},
  pages={111},
  year={2024},
  publisher={Springer}
}

@article{rashidi2025predicting,
  title={Predicting the Price of Gold in the Financial Markets Using Hybrid Models},
  author={Rashidi, Mohammadhossein and Modarres, Mohammad},
  journal={arXiv preprint arXiv:2505.01402},
  year={2025}
}

\newpage
\section*{Appendix: A Possible Approach When the Trend Predictor Is Unavailable}

The TATS model is designed to leverage the trend of a time series to reduce forecasting errors. A natural question arises: what if the TATS model cannot be implemented because the trend predictor's accuracy is insufficient or a reliable trend predictor is not achievable? Even in such cases, incorporating trend information remains important. The following example illustrates this concept in a simple setting.
 \\

\textbf{Example 1.} Let $\{y_t\}_{t=1,2,3,4,5}$ be a time series with values $\{7,5,9,7,8\}$. Consider two time series models for forecasting $\{y_t\}$: the first model generates forecasts $\{\hat{y}^1_t\}_{t=1,2,3,4,5}$ with values $\{8,2,14,6,10\}$, while the second model generates forecasts $\{\hat{y}^2_t\}_{t=1,2,3,4,5}$ with values $\{6,8,4,8,6\}$, as illustrated in Fig.~\ref{fig:appendix_example1}. As can be seen, the MSE of both models is the same. However, the trend detection accuracy (TDA), as defined in (\ref{eq:TDA}), differs substantially: the first model achieves a TDA of 1, whereas the second model has a TDA of 0. If evaluation is based solely on the MSE, the two models would appear equivalent. In contrast, the TDA reveals that the models are not equivalent, and that the first model, unlike the second, successfully captures the trend of the actual time series. For example, in the financial domain, if the actual time series represents the price of a stock, an investor who buys and sells based on the first model would achieve significantly better outcomes compared to using the second model. 
To illustrate this, suppose the investor holds some stocks at time step $t=2$ and must choose between two options: (1) buy more stock at $t=2$ and sell at $t=3$, or (2) sell all stocks now at $t=2$. Model~1 indicates that the stock price will increase from $t=2$ to $t=3$, suggesting that buying more and selling later at $t=3$ is profitable. In contrast, Model~2 predicts a price decrease, suggesting that selling immediately is optimal. The actual data show that the price increases from $t=2$ to $t=3$, so following Model~1 results in a gain. 
In another domain, if the actual time series represents health-related data, Model~1 would be more reliable for decision-making by a medical team. This indicates that, although both models have the same MSE, Model~1 is superior in practice, as it better preserves the trend information of the actual time series compared to Model~2.

\begin{figure}[H]
    \centering
    \includegraphics[width=0.8\linewidth]{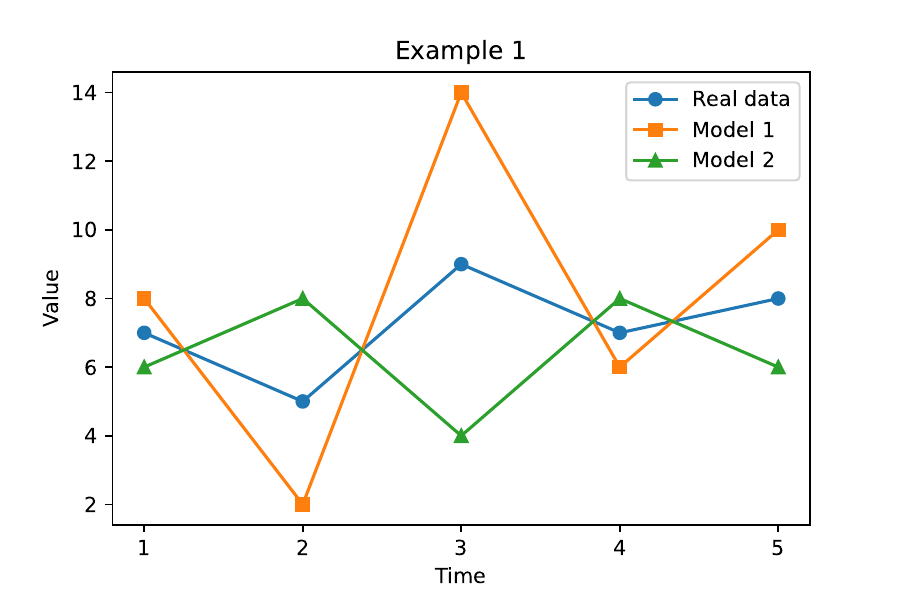}
    \caption{The actual time series (real data) and the forecasts generated by two different models in Example~1. Both models have the same MSE; however, Model~1 successfully captures the time series trend (directional movements).}

    \label{fig:appendix_example1}
\end{figure}

Example~1 demonstrates that incorporating trend information can lead to time series forecasts that are more useful for decision-making. Loss functions commonly used in time series modeling, which focus solely on forecasting error, fail to distinguish between Model~1 and Model~2 in Example 1. Consequently, a possible approach is to employ trend-aware loss functions that penalize forecasts moving in the wrong trend or direction. Specifically, one may consider a trend-aware loss function of the form
\[
\sum_t (y_t - \hat{y}_t)^2 + \gamma_t,
\]
where $\gamma_t$ is a penalty applied when the forecast moves in the wrong direction, that is, when
$
(\hat{y}_t - y_{t-1})(y_t - y_{t-1}) < 0.
$
If this trend-aware loss function is applied to the two models in Example~1, the first model is preferred, as it achieves a lower loss by correctly capturing the trend of the time series, whereas the second model fails to do so. It should be noted that, unlike the TATS model—which guarantees a reduction in forecasting error—trend-aware loss functions in time series models may worsen the forecasting accuracy. Nevertheless, when constructing the TATS model is not feasible, employing a trend-aware loss function to incorporate trend information remains a possible approach for further investigations and future studies.

\end{document}